\documentclass[11pt]{article}

\usepackage{times}
\usepackage[utf8]{inputenc}
\usepackage{algorithm}
\usepackage{algorithmic}
\usepackage{natbib}
\usepackage{amsthm}
\usepackage{amssymb}
\usepackage{bbm}
\usepackage{amsmath}
\usepackage{graphicx}
\usepackage{xcolor}
\usepackage{hhline}
\usepackage[breaklinks=true]{hyperref}
\usepackage{breakcites}

\hypersetup{
    colorlinks,
    linkcolor={blue!50!black},
    citecolor={blue!50!black},
}
\colorlet{linkequation}{blue}

\newtheorem{thm}{Theorem}[section]

\newtheorem{lem}[thm]{Lemma}

\newtheorem{rem}[thm]{Remark}

\newcommand{\remp}{R_{\rm emp}}
\newcommand{\rpop}{R_{\rm pop}}
\newcommand{\be}{\mathbf{e}}
\newcommand{\E}{\mathbb{E}}
\renewcommand{\P}{\mathbb{P}}
\newcommand{\sign}{{\rm sign}}

\newcommand{\Dcal}{\mathcal{D}}
\newcommand{\Acal}{\mathcal{A}}
\newcommand{\Scal}{\mathcal{S}}
\newcommand{\Xcal}{\mathcal{X}}
\newcommand{\Ycal}{\mathcal{Y}}
\newcommand{\Zcal}{\mathcal{Z}}

\advance \oddsidemargin by -0.8in
\newcommand{\myparskip}{3pt}
\parskip \myparskip

\title{\textbf{A Tight Lower Bound for Uniformly Stable Algorithms}}
\author{
  Qinghua Liu \\
  Princeton University\\
  \texttt{qinghual@princeton.edu}
  \and  Zhou Lu\footnote{This work is done during LZ's visit to SQZ institution.} \\
  Princeton University\\
  \texttt{zhoul@princeton.edu}
}
\date{December, 2020}

\begin{document}

\maketitle


\begin{abstract}
    Leveraging algorithmic stability to derive sharp generalization bounds is a classic and powerful approach in learning theory. Since \cite{vapnik1974theory} first formalized the idea for analyzing SVMs, it has been utilized to study many fundamental learning algorithms (e.g., $k$-nearest neighbors \citep{rogers1978finite}, stochastic gradient method \citep{hardt2016train}, linear regression \citep{maurer2017second}, etc). In a recent line of great works by \cite{feldman2018generalization,feldman2019high} and \cite{bousquet2020sharper}, they prove a high probability generalization upper bound of order $\widetilde{\mathcal{O}}(\gamma  +\frac{L}{\sqrt{n}})$ for any uniformly $\gamma$-stable algorithm and $L$-bounded loss function. Although much progress was achieved in proving 
    generalization upper bounds for stable algorithms, our knowledge of lower bounds is rather limited. In fact, there is no nontrivial lower bound known ever since the study of uniform stability \citep{bousquet2002stability}, to the best of our knowledge. In this paper we fill the gap by proving a tight  generalization lower bound of order $\Omega(\gamma+\frac{L}{\sqrt{n}})$, which matches the best known upper bound up to logarithmic factors.

\end{abstract}


\section{Introduction}
Estimating the generalization error of learning algorithms is at the heart of  modern  statistical learning theory. 
One classic approach is to control the generalization error via notions of model complexity, which has been extensively studied for decades \citep{vapnik2013nature}. However, as the saying goes "It's hard to please all", analysis of model complexity doesn't always give satisfactory answers to all learning algorithms. 
For example, when analyzing stochastic gradient descent on convex Lipschitz functions, one cannot obtain meaningful generalization bounds by proving uniform convergence for all empirical risk minimizers   \citep{shalev2010learnability,feldman2016generalization}.

Another classic way for proving generalization bounds is to utilize the stability of algorithms, pioneered by \cite{vapnik1974theory}, \cite{rogers1978finite}, \cite{devroye1979distribution1, devroye1979distribution} and further studied in \cite{lugosi1994posterior,bousquet2002stability,  mukherjee2006learning, shalev2010learnability, hardt2016train,maurer2017second}, etc. Intuitively, stability measures the sensitivity of a learning algorithm to the change of a single data point in the training set. Stronger stability often guarantees better generalization, as the learning algorithm is robust to small perturbation of data.

In this paper, we study the generalization error of uniformly stable algorithms which were first introduced by \cite{bousquet2002stability}. 
Formally, we consider the following learning problem where we are given a training set $\mathcal{S}=\{(x_1,y_1),...,(x_n,y_n)\}$ consisting of $n$ i.i.d. samples from some unknown distribution $\mathcal{D}$ on domain $\mathcal{Z}\subset \mathcal{X} \times \mathcal{Y}$. A learning algorithm 
$\mathcal{A}:\mathcal{Z}^n\to \mathcal{Y}^\mathcal{X}$ is a function which maps a training set to a function mapping from instance space $\mathcal{X}$ into label space $\mathcal{Y}$. 
We denote by ${\mathcal{A}}_{\mathcal{S}}\in\mathcal{Y}^\mathcal{X}$ the output function mapping obtained by feeding algorithm $\mathcal{A}$ with training set $\mathcal{S}$.

We measure the performance of $\mathcal{A}_\mathcal{S}$ by a non-negative loss function $\ell: \mathcal{Y}\times \mathcal{Y} \to \mathbb{R}$, and define its population risk  as
\begin{equation}
    \rpop (\mathcal{A}_\mathcal{S})=\E_{(x,y)\sim \Dcal} \left[\ell(\mathcal{A}_\mathcal{S}(x),y)\right],
\end{equation}
as well as its empirical risk as
\begin{equation}
    \remp (\mathcal{A}_\mathcal{S})=\frac{1}{n}\sum_{i=1}^n \ell(\mathcal{A}_\mathcal{S}(x_i),y_i).
\end{equation}

One classic approach to controlling the generalization error $\rpop(\mathcal{A}_\mathcal{S})-\remp(\mathcal{A}_\mathcal{S})$ is by restricting the sensitivity of algorithm $\Acal$ to changes in training set $\Scal$ (e.g., removing or modifying one of the data points). 
In order to quantify the sensitivity of algorithms, \cite{vapnik1974theory,bousquet2002stability} develop the notion of  stability. 
Formally, a learning algorithm $\Acal$ is called uniformly $\gamma$-stable 
(we will use 'stable' as a shorthand for 'uniformly stable' throughout this paper) \citep{bousquet2002stability} if for any $\mathcal{S}=\{(x_1,y_1),...,(x_n,y_n)\}\in (\Xcal\times\Ycal)^n$, $\mathcal{S}^i=\{(x_1,y_1),...,(x_{i-1},y_{i-1}),(x'_i,y'_i),(x_{i+1},y_{i+1}),...,(x_n,y_n)\}\in (\Xcal\times\Ycal)^n$ and any $(x,y)\in \mathcal{X} \times \mathcal{Y}$, we have
\begin{equation}\label{sta}
    |\ell(\mathcal{A}_\mathcal{S}(x),y)-\ell(\mathcal{A}_{\mathcal{S}^i}(x),y)|\le \gamma.
\end{equation}

Many generalization bounds have been proved via the notion of stability (e.g., \cite{bousquet2002stability,feldman2018generalization,feldman2019high,bousquet2020sharper}) and the current best one  is given by \cite{bousquet2020sharper}. Specifically, \cite{bousquet2020sharper} prove a general moment inequality and use it as a tool to derive a sharp $\mathcal{O}(\gamma \log n \log \frac{1}{\delta} +\frac{L}{\sqrt{n}}\sqrt{\log \frac{1}{\delta}})$  bound for any $\gamma$-stable algorithm and $L$-bounded loss function. They also provide an almost matching lower bound for the moment inequality. However, it remains unclear whether this moment lower bound can further imply a lower bound for generalization error.


Although much progress has been achieved in proving generalization upper bounds for stable algorithms, our knowledge of lower bounds is rather limited. In fact, there is no nontrivial lower bound known ever since the study of uniform stability  \citep{bousquet2002stability}, to the best of our knowledge. In this paper we fill the gap by proving a tight  generalization lower bound of order $\Omega(\gamma+\frac{L}{\sqrt{n}})$, which matches the best known upper bound up to logarithmic factors.

\begin{thm}[informal]\label{thm:main}
	There exist domain $\Xcal\times\Ycal$, distribution $P$ over $\Xcal\times\Ycal$, $L$-bounded loss function $\ell$, and $\gamma$-stable algorithm $\Acal$ such that with constant probability over the random drawing of $\mathcal{S}$, the output  function mapping $\mathcal{A}_\mathcal{S}$ has generalization error $\Omega(\gamma+\frac{L}{\sqrt{n}})$.
\end{thm}
To the best of our knowledge, Theorem \ref{thm:main} provides the first nontrivial and almost matching generalization lower bound for uniformly stable algorithms and  therefore deepens our understanding of the methodology of algorithmic stability. 

\subsection{Review of upper bounds}
In the seminal work by \cite{bousquet2002stability}, they provide the first generalization upper bound that holds for any $\gamma$-stable algorithm and $L$-bounded loss function. Specifically, they prove that
with probability at least $1-\delta$,
\begin{equation}
    \rpop(\mathcal{A}_\mathcal{S})-\remp(\mathcal{A}_\mathcal{S})=\mathcal{O}\left(\sqrt{n} \gamma \sqrt{\log \frac{1}{\delta}} +\frac{L}{\sqrt{n}}\sqrt{\log \frac{1}{\delta}}\right).
\end{equation}
However, its dependence on $n$ is  suboptimal in that its tightness is guaranteed only  when $\gamma=\mathcal{O}(\frac{1}{n})$. 
This upper bound was recently improved by \cite{feldman2018generalization, feldman2019high}, who show 
\begin{equation}\label{ntightu}
    \rpop(\mathcal{A}_\mathcal{S})-\remp(\mathcal{A}_\mathcal{S})=\mathcal{O}\left(\gamma (\log n)^2 + \gamma \log n \log \frac{1}{\delta} +\frac{L}{\sqrt{n}}\sqrt{\log \frac{1}{\delta}}\right).
\end{equation}
The improvement is significant because they remove the $\sqrt{n}$ term so that the rate is optimal as long as 
$\gamma=\mathcal{O}(\frac{1}{\sqrt{n}})$.
In the latest work by \cite{bousquet2020sharper}, this upper bound was further sharpened to 
\begin{equation}\label{tightu}
    \rpop(\mathcal{A}_\mathcal{S})-\remp(\mathcal{A}_\mathcal{S})=\mathcal{O}\left(\gamma \log n \log \frac{1}{\delta} +\frac{L}{\sqrt{n}}\sqrt{\log \frac{1}{\delta}}\right),
\end{equation}
which removes the unnecessary $\mathcal{O}(\gamma (\log n)^2)$ term in \eqref{ntightu} with a simpler proof. 
In \cite{bousquet2020sharper}, they prove a general moment inequality for weakly correlated random variables, and derive \eqref{tightu} as a corollary.

\subsection{Other related works}

The notion of stability was first used in analyzing hard-margin SVMs \citep{vapnik1974theory}, which was later followed by \cite{rogers1978finite, devroye1979distribution, devroye1979distribution1} to prove generalization bounds for $k$-nearest neighbors. Other early works mostly focus on specific learning problems by extending their techniques \citep{devroye2013probabilistic}. \cite{bousquet2002stability} first prove general results on the relationship between stability and generalization.  They introduce the notion of uniform stability and provide various generalization bounds based on different notions of stability.

As for recent studies on stability, \cite{hardt2016train} prove generalization bounds for stochastic gradient descent using uniform stability. \cite{maurer2017second} study linear regression with a strongly convex regularizer and a sufficiently smooth loss function. \cite{bousquet2020proper} prove tight exponential upper bounds for the SVM in the realizable setting. And \cite{shalev2010learnability} prove that by adding a strongly convex term to the objective, ERM solutions to convex learning problems can be made uniformly stable.

Uniform stability also has close relationship with differential privacy \citep{dwork2008differential}. For example, a uniformly stable learning algorithm can be transformed into a differentially private one by adding noise to the output \citep{dwork2018privacy}. 

\section{Preliminaries}
While various concentration arguments play a vital role in proving upper bounds, to construct hard cases for lower bounds we will need anti-concentration instead. In this section, we introduce some basic anti-concentration inequalities that will be used in our proof.

\begin{lem}[Paley–Zygmund inequality]\label{pal}
Let $Z\ge 0$ be a random variable with bounded second moment. For all $\theta\in[0, 1]$,  we have 
\begin{equation}
    \P(Z>\theta \E(Z))\ge (1-\theta)^2 \frac{\E[Z]^2}{\E[Z^2]}.
\end{equation}
\end{lem}
\begin{proof}
We decompose $\E[Z]$ as
\begin{equation}
    \E[Z]=\E[Z\times 1_{Z\le \theta \E[Z]}]+\E[Z\times 1_{Z> \theta \E[Z]}].
\end{equation}
The first term is upper bounded by $\theta \E[Z]$, and the second term is at most $\sqrt{\E[Z^2] \P(Z>\theta \E[Z])}$ by Cauchy–Schwarz inequality. The desired inequality thus follows.
\end{proof}

Paley–Zygmund inequality implies that if a non-negative random variable has relatively small variance (so that its standard deviation and mean are of the same order), then with constant probability the random variable and its mean are within the same order of magnitude. 
Below we utilize Paley–Zygmund inequality to prove an anti-concentration inequality for sum of Rademacher random variables.


\begin{lem}[anti-concentration of sum of Rademacher random variables]\label{imp}
Let $X_1,\ldots,X_n$ be independent Rademacher random variables. Then
\begin{equation}
    \P(\sum_{i=1}^{n} X_i > \frac{\sqrt{n}}{2}) \ge \frac{3}{32}.
\end{equation}
\end{lem}
\begin{proof}
Define $S=\sum_{i=1}^{n} X_i$. We have that $\forall i\ne j$
\begin{equation}
    \E [X_i^2]=\E [X_i^4]=\E [X_i^2 X_j^2]=1,\quad \E [X_i X_j]=\E [X_i^3 X_j]=0.
\end{equation}

Therefore $\E[S^2]=n$ and $\E[S^4] =n+3n(n-1)\le 3n^2$. By Paley–Zygmund inequality (Lemma \ref{pal}), we have
    $\P(S^2> n/{4}) \ge {3}/{16}.$
Noting that the distribution of $S$ is symmetric, we conclude 
    $\P(S> {\sqrt{n}}/{2}) \ge {3}/{32}$.
\end{proof}

Lemma \ref{imp} shows that the sum of $n$ independent Rademacher random variables has  absolute value $\Omega(\sqrt{n})$ with constant probability. This lemma will play an important role in establishing the $\frac{L}{\sqrt{n}}$ term in our lower bound.

\section{Main Result}
In this section, we present our main result which constructs a hard case such that with constant probability,  the $\gamma$-stable learning algorithm $\mathcal{A}$ we design has generalization error of order $\Omega(\gamma+\frac{L}{\sqrt{n}})$, 
which matches the best known upper bound in \eqref{tightu} up to logarithmic factors.


\begin{thm}[lower bound]\label{main}
	For any $0<\gamma\le L$ and $n\in\mathbb{N}$, 
	there exist domain $\Xcal\times\Ycal$, distribution $P$ over $\Zcal\subset\Xcal\times\Ycal$, $L$-bounded loss function $\ell$, and $\gamma$-stable algorithm $\Acal$ such that given a training set $\Scal$ consisting of $n$ i.i.d. samples from $P$, with probability at least ${3}/{64}$, 
	\begin{equation}
   \rpop(\mathcal{A}_\mathcal{S})-\remp(\mathcal{A}_\mathcal{S})\ge \frac{\gamma}{4}+\frac{L}{32\sqrt{n}}.
\end{equation}
\end{thm}

\begin{proof}
At a high level, we construct $\mathcal{X}$ to be the collection of base vectors in $\mathbb{R}^d$ where $d\gg n$ and $P$ being the uniform distribution, so that with high probability $\mathcal{S}$ contains vectors vertical to each other. We then add a negative copy of each vector and construct $\mathcal{A}$ to be a '$\gamma$ majority vote' so that $\mathcal{A}$ is as poor as random guess on population, but performs slightly better than random guess on $\mathcal{S}$ which brings a $\Omega(\gamma)$ gap in generalization error. Then we extend half of the vectors' length by twice, so that by using an anti-concentration bound, with constant probability $\mathcal{S}$ samples $\Omega(\sqrt{n})$ more 'shorter' vectors, which further decreases $\remp$ by $\Omega(\frac{L}{\sqrt{n}})$.

To begin with, we introduce the construction of our hard case.
 Given sample size $n$, sensitivity parameter $\gamma>0$ and  boundedness parameter $L\ge\gamma$, we set $d:=4 n^2$ and construct
\begin{equation}
    \mathcal{X}:=\{L \sigma_1 \be_1,-L \sigma_1 \be_1,...,L \sigma_d \be_d,-L \sigma_d \be_d\},
\end{equation}
where $\sigma_i=1+\mathbbm{1}_{[i>\frac{d}{2}]}$. Furthermore, let $\mathcal{Y}:=\mathcal{X}$ and $\mathcal{Z}:=\{(x,x)|x\in \mathcal{X} \}$ so that the label $y$ of each $x\in \mathcal{X}$ is itself. We choose $P$  to be the uniform distribution over $\mathcal{Z}$, and use the $\ell_1$-norm loss function, i.e., $\ell(y,\hat{y}):=||y-\hat{y}||_1$. 

Given training set $\mathcal{S}:=\{(x_1,y_1),...,(x_n,y_n)\}$, our learning algorithm is defined as
\begin{equation}
    \mathcal{A}_\mathcal{S}(\pm L \sigma_i \be_i):= \sign \Big[\big(\sum_{j=1} ^n x_j\big)_i\Big] \gamma \sigma_i \be_i,
\end{equation}
where $(z)_i$ denotes the $i^{\rm th}$ coordinate of $z$. It is easy to check that our learning algorithm $\Acal$ is $4 \gamma$-stable, and the loss function $\ell$ is upper bounded by $4L$ over $\Ycal\times\Ycal$. 

In the remainder of this section, we will prove the generalization error of algorithm $\Acal$ is lower bounded by $\Omega(\gamma+\frac{L}{\sqrt{n}})$ with constant probability. 
Specifically, the proof consists of two parts, where the first part aims to compute the population loss $\rpop(\mathcal{A}_\mathcal{S})$ exactly, and the second one provides an upper bound for the empirical (training) loss $\remp(\mathcal{A}_\mathcal{S})$.

\hspace{-6mm}\textbf{Part 1: compute $R_{\rm pop}(\Acal_S)$ exactly}

We observe that 
\begin{align*}
     \rpop(\mathcal{A}_\mathcal{S})& = \E_{(x,y)\sim \Dcal} \left[\ell(\mathcal{A}_\mathcal{S}(x),y)\right]\\
    &=\frac{ 1}{2d}\sum_{i=1}^d \left(\|\mathcal{A}_\mathcal{S}(L\sigma_i \be_i)-L\sigma_i \be_i\|_1 + 
    \|\mathcal{A}_\mathcal{S}(-L\sigma_i \be_i)+L\sigma_i \be_i\|_1\right). 
\end{align*}
Notice $\mathcal{A}_\mathcal{S}(-L\sigma_i \be_i)\equiv\mathcal{A}_\mathcal{S}(L\sigma_i \be_i)$ always lies on the line segment between 
$L\sigma_i \be_i$ and $-L\sigma_i \be_i$ by the definition of $\Acal$ and $\gamma\le L$.
As a result, we have  
$$\|\mathcal{A}_\mathcal{S}(L\sigma_i \be_i)-L\sigma_i \be_i\|_1 + 
    \|\mathcal{A}_\mathcal{S}(-L\sigma_i \be_i)+L\sigma_i \be_i\|_1\equiv 2L \sigma_i ,$$
 which directly implies 
 \begin{equation}\label{eq1}
   \rpop(\mathcal{A}_\mathcal{S})\equiv \frac{3L}{2}. 	
 \end{equation}


\hspace{-6mm}\textbf{Part 2: upper bound $R_{\rm emp}(A_S)$}

To proceed, we define two useful events: event $E_1$ that any two different $x_i,x_j$ in $\mathcal{S}$ are orthogonal to each other, and event $E_2$ that there are at least 
$\sqrt{n}/2$ more $x_i$'s with norm $L$ than those with norm $2L$ in $\Scal$.
For notational convenience, we further define $\sigma^{(i)}:=||x_i||_1/L$.

Conditioning on $E_1$, we have
\begin{equation}\label{eq2}
 	  \remp(\mathcal{A}_\mathcal{S}) =\frac{1}{n}\sum_{(x,y)\in S} \ell(\mathcal{A}_\mathcal{S}(x),y)
    =\frac{L-{\gamma}}{n} \sum_{i=1}^n \sigma^{(i)}.
 \end{equation}
On the other hand, conditioning on $E_2$, we have
\begin{equation}\label{eq3}
 	  \frac{1}{n} \sum_{i=1}^n \sigma^{(i)}  \le \frac{1}{n}\E\left[\sum_{i=1}^n \sigma^{(i)}\right]-\frac{1}{8\sqrt{n}}
    =\frac{3}{2}-\frac{1}{8\sqrt{n}}.
 \end{equation}
 Combining equations \eqref{eq1}, \eqref{eq2} and \eqref{eq3}, we obtain the desired lower bound
 \begin{equation*}
 	\remp(\mathcal{A}_\mathcal{S}) 
 	= \frac{L-{\gamma}}{n} \sum_{i=1}^n \sigma^{(i)}
 	\le ({L-{\gamma}}) \left(\frac{3}{2}-\frac{1}{8\sqrt{n}}\right)
 	\le \rpop(\mathcal{A}_\mathcal{S}) 
 	- \frac{L}{8\sqrt{n}}-\gamma.
 \end{equation*}

Now, the only thing left is to estimate $\P(E_1\bigcap E_2)$.
By Lemma \ref{imp}, we have
$\P(E_2)\ge 3/32$. Moreover, note that 
\begin{equation}
    \P(E_1|E_2)\ge (1-\frac{n}{0.5 d})^n
    =(1-\frac{1}{2n})^n \ge 1-\frac{1}{2}=\frac{1}{2}.
\end{equation}
Therefore, we obtain $\mathbb{P}(E_1\bigcap E_2)=\P(E_1|E_2) \P(E_2)\ge {3}/{64}$.
Finally, rescaling $\gamma$ and $L$ by $1/4$ concludes the whole proof. 
\end{proof}



\begin{rem}
We can also avoid analyzing the relationship between events $E_1$ and $E_2$ by setting $d$ large enough and taking a union bound on $\neg E_1$ and $\neg E_2$.
\end{rem}


Theorem \ref{main} directly implies that it is impossible to achieve $o(\gamma+\frac{L}{\sqrt{n}})$ generalization error in general  and the  upper bound in \citep{bousquet2020sharper} is almost optimal.
We comment that our lower bound here holds with constant probability and it would be interesting to generalize it to the high-probability regime so that it can also reveal the dependence on the failure probability.
And to do that, the first step might be to replace Lemma \ref{imp} with a stronger anti-concentration  inequality that can handle relatively small probability $\delta\ll 1$.

\section{Conclusion}
In this paper we prove a tight $\Omega(\gamma+\frac{L}{\sqrt{n}})$ generalization lower bound for uniformly stable algorithms, which matches the best known upper bound in \citep{bousquet2020sharper} up to logarithmic factors. To the best of our knowledge, this result provides the first matching lower bound which has been unknown for more than a decade since the first upper bound was given in \citep{bousquet2002stability}, thus greatly complementing our knowledge about the limit of this classic methodology.

\bibliography{Xbib}
\bibliographystyle{plainnat}

\end{document}